%% file: rblda.tex
\documentclass[preprint,3p,twocolumn]{elsarticle}
\usepackage{amsthm,amsmath,amssymb,amsfonts}
\usepackage{IEEEtrantools}
\usepackage{multirow}
\usepackage{algorithmic}
\usepackage{algorithm}
\usepackage{url}
\usepackage{booktabs} 
\usepackage{graphicx,color}

\usepackage[colorlinks, citecolor=blue, linkcolor=black]{hyperref}
\usepackage[numbers]{natbib}
\usepackage{enumerate}
\usepackage{times}
\usepackage[OT1]{fontenc}
\input{general_setting}

\begin{document}

\begin{frontmatter}
\title{Regularized Bilinear Discriminant Analysis for Multivariate Time Series Data}
\author[J. Zhao]{Jianhua Zhao\corref{cor1}} \ead{jhzhao.ynu@gmail.com}
\cortext[cor1]{Corresponding author at: School of Statistics and Mathematics, Yunnan University of Finance and Economics, China.}
\author[J. Zhao]{Haiye Liang} \ead{hyliang0209@qq.com}
\author[S. Li]{Shulan Li} \ead{shulanli0526@qq.com}
\author[J. Zhao]{Zhiji Yang} \ead{yangzhiji@ynufe.edu.cn}
\author[Z. Wang]{Zhen Wang} \ead{zhenwang0@gmail.com} 
\address[J. Zhao]{School of Statistics and Mathematics, Yunnan University of Finance and Economics, Kunming, 650221, China.} 
\address[S. Li]{School of Accounting, Yunnan University of Finance and Economics, Kunming, 650221, China.}
\address[Z. Wang]{The Center for OPTical IMagery Analysis and Learning (OPTIMAL) and School of Mechanical Engineering, Northwestern Polytechnical University, Xi’an 710072, China}

\begin{abstract}
	In recent years, the methods on matrix-based or bilinear discriminant analysis (BLDA) have received much attention. Despite their advantages, it has been reported that the traditional vector-based regularized LDA (RLDA) is still quite competitive and could outperform BLDA on some benchmark datasets. Nevertheless, it is also noted that this finding is mainly limited to image data. In this paper, we propose regularized BLDA (RBLDA) and further explore the comparison between RLDA and RBLDA on another type of matrix data, namely multivariate time series (MTS). Unlike image data, MTS typically consists of multiple variables measured at different time points. Although many methods for MTS data classification exist within the literature, there is relatively little work in exploring the matrix data structure of MTS data. Moreover, the existing BLDA can not be performed when one of its within-class matrices is singular. To address the two problems, we propose RBLDA for MTS data classification, where each of the two within-class matrices is regularized via one parameter. We develop an efficient implementation of RBLDA and an efficient model selection algorithm with which the cross validation procedure for RBLDA can be performed efficiently. Experiments on a number of real MTS data sets are conducted to evaluate the proposed algorithm and compare RBLDA with several closely related methods, including RLDA and BLDA. The results reveal that RBLDA achieves the best overall recognition performance and the proposed model selection algorithm is efficient; Moreover, RBLDA can produce better visualization of MTS data than RLDA. 
	
\end{abstract}
\end{frontmatter}	
\section{Introduction}
\label{sec:intro}

Principal component analysis (PCA) and Fisher linear discriminant analysis (LDA) are widely used dimension reduction methods for vector data, where observations are vectors. To perform dimension reduction on matrix data, where observations are matrices, one could consider first vectorizing the matrix data and then applying PCA and LDA or their variants to the resulting vector data. However, the vectorization destroys the natural matrix data structure and may lose potentially useful information between rows or columns \cite{jpye-glram}. Furthermore, the vectorized data is often very high-dimensional, on which the vector-based methods may suffer from the so-called curse of dimensionality \cite{zhao2012-bppca}.

To deal with the high-dimensional problem for image data, instead of employing vectorization, several matrix-based dimension reduction methods have been proposed, such as bilinear PCA (BPCA) \citep{zhangdq-2dpca}, bilinear discriminant analysis (BLDA) \cite{Noushath-2dlda,zhao2012-slda}, etc. The common feature is that these methods use bilinear transformations, rather than the linear ones in PCA and LDA. As a result, their numbers of parameters are much less than those in PCA and LDA and can alleviate greatly the high-dimensional problem and significantly reduce computational costs \citep{jpye-glram}.
In essence, the matrix-based methods such as BPCA and BLDA \cite{Noushath-2dlda,zhao2012-slda} can be viewed as a special case of the vector-based methods PCA and LDA under the assumption that the transformation matrix is separable \cite{zhao2012-bppca,zhao2012-slda}. When this assumption tends to hold approximately, it would be expected that matrix-based methods can perform favorably in finite matrix data applications since they simply require specifying a much smaller number of parameters. For PCA-related methods, even if the assumption fails to hold true, it is shown in \cite{zhao2012-bppca} that the matrix-based probabilistic BPCA could obtain better performance than vector-based probabilistic PCA on datasets with small sample sizes due to the variance-bias trade-off. However, this superiority does not always hold true for LDA-related methods. The comprehensive study in \citep{zheng-2dlda} shows that regularized LDA (RLDA) outperforms BLDA on a number of real image datasets. Similar results can be found in \citep{zhao2015-rlda-2stage}.

Different from \citep{zheng-2dlda,zhao2015-rlda-2stage}, where the comparisons are mainly limited to image data, this paper further explores the comparison between vector-based LDA and matrix-based BLDA on another type of matrix data, namely multivariate time series (MTS), whose observation is usually the result of multiple variables measured at different time points, and can be naturally represented by a matrix. MTS data classification is a hot research topic in time series data mining. A good summary on the development in this field can be found in \citep{handhika2019multivariate}. Many methods have been proposed in the literature, including discriminant analysis \cite{maharaj2014discriminant}, hypothesis testing \citep{maharaj1999comparison}, singular value decomposition (SVD) \citep{li2006real}, common PCA \citep{li2016accurate}, dynamic time warping (DTW) \citep{holt-mdtw2007,ruiz2020great}, Mahalanobis distance-based DTW \citep{mei2015learning}, etc. However, all of these methods fail to consider the matrix data structure inherent in MTS data. Therefore, there is a gap between the active area on matrix-based methods and MTS data classification.

Unlike the matrix data of images, where both the row and column are variables (i.e., pixels), the salient characteristic of MTS data is that the row and column are different: one is variable and the other is time. Intuitionally, it seems that the assumption of a separable transformation matrix in BPCA and BLDA is more suitable for such data. In fact, BPCA has been suggested in \cite{zhangdq-2dpca} and the experiments show that BPCA is often advantageous over the closely related methods. Since BPCA is an unsupervised method, BLDA is further suggested in \citep{zhao2021-pblda} to utilize the matrix data structure and label information simultaneously. The results show that the use of label information is generally beneficial to the classification of MTS data. Since the existing implementations of BLDA \cite{Noushath-2dlda,inoue-2dlda} suffers from two problems: (i) BLDA can not be performed when one of its within-class matrices is singular; (ii) the computational cost could be very heavy when the number of variables or time points is high, a variant of BLDA based on pseudo-inverse (PBLDA) is also proposed in \citep{zhao2021-pblda}.


As detailed in \citep{friedman-rda}, pseudo-inverse may produce biased estimates of the eigenvalues of the within-class matrix; the smallest ones are biased towards values that are too low, which may cause the interesting discriminant directions to be fooled by the subspace associated with the low eigenvalues of the within-class matrix. To mitigate this problem, RLDA \citep{friedman-rda} has been proven effective  \citep{jpye-rlda,zhzhang-rlda}, which stabilizes the within-class matrix by adding a small multiple of the identity matrix. In this paper, to address problem (i) we extend RLDA to matrix data and propose regularized BLDA (RBLDA) for MTS data classification, where each of the two within-class matrices is regularized via one parameter. To address problem (ii), we develop an efficient implementation of RBLDA. In addition, to choose the two regularization parameters, we develop an efficient model selection algorithm so that the cross validation procedure for RBLDA can be performed efficiently. With the proposed RBLDA, we conduct a comprehensive comparison between vector-based RLDA and matrix-based RBLDA on MTS data.



The remainder of the paper is organized as follows. \refs{sec:rev} gives a brief review of linear discriminant analysis (LDA), regularized LDA (RLDA), and the existing implementation of BLDA. \refs{sec:rblda} proposes our RBLDA and its efficient model selection algorithm. 
\refs{sec:expr} constructs an empirical study to evaluate the proposed algorithm and compare RBLDA with several closely related competitors. We end the paper with some conclusions and discussions in \refs{sec:con}.

Notations: in the sequel, the transpose of vector/matrix is denoted by the superscript
$'$, and the $d\times d$ identity matrix by $\bI_d$. Moreover,
$\tr(\bA)$ is the trace of matrix $\bA$ and $\bB=\bdiag{(\bA_1,\dots,\bA_m)}$ stands for the block diagonal matrix formed by aligning the input matrices $\bA_1,\dots,\bA_m$ along the diagonal of $\bB$.

\section{Review of LDA and Related Variants}\label{sec:rev}

\subsection{Linear discriminant analysis (LDA)}\label{sec:lda}
Given a set of $d$-dimensional vector-valued data $\{\bx_{i}\}_{i=1}^n$ consisting of $c$ classes, let $\bm_k=\frac1{n_k}\sum\nolimits_{i\in\cC_k}\bx_i$ be the sample mean of class $k$, $\cC_k$, and $\bm=\frac1n\sum\nolimits_{i=1}^n\bx_i$ be the sample mean, then the between-class, within-class and total scatter matrices are given by
\begin{IEEEeqnarray}{rCl}
\bS_b&=&\frac1n\sum\nolimits_{k=1}^cn_k(\bm_k-\bm)(\bm_k-\bm)',\label{eqn:lda.Sb}\\
\bS_w&=&\frac1n\sum\nolimits_{k=1}^c\sum\nolimits_{i\in\cC_k}(\bx_i-\bm_k)(\bx_i-\bm_k)',\nonumber\\ 
\bS_t&=&\frac1n\sum\nolimits_{i=1}^n(\bx_i-\bm)(\bx_i-\bm)'.\label{eqn:lda.St}
\end{IEEEeqnarray} 
Consider a linear transformation $\by=\bV_w'\bx$, where $\bV_w\in\R^{d\times q}$ and $q<d$. The within-class and between-class scatter matrices in $\by$-space are $\bV_w'\bS_w\bV_w$ and $\bV_w'\bS_b\bV_w$, respectively. The Fisher criterion aims to find $\bV_w$ that maximizes $\cF$ as \cite{fukunaga-1990},
\begin{equation*}
\cF=\mathop{\hbox{max}}\nolimits_{\bV_w}\hbox{tr}\left\{(\bV_w'\bS_w\bV_w)^{-1}(\bV_w'\bS_b\bV_w)\right\}.\label{eqn:lda.obj}
\end{equation*}
It can be seen that $\bV_w$ can only be determined up to a $q\times q$ nonsingular transformation. A common solution is obtained by solving the problem
\begin{IEEEeqnarray}{rCl}
\mathop{\hbox{max}}\nolimits_{\bV_w}\tr\left\{\bV_w'\bS_b\bV_w\right\},\quad s.t.\,\, \bV_w'\bS_w\bV_w=\bI_q,\label{eqn:lda.obj.w}
\end{IEEEeqnarray}
and the columns of $\bV_w$ are $\bS_w$-orthogonal due to the constraint $\bV_w'\bS_w\bV_w=\bI_q$.

Alternatively, $\bS_w$ in \refe{eqn:lda.obj.w} can be replaced by $\bS_t$ \cite{zhzhang-rlda,jpye-rlda} and the optimization problem is 
\begin{IEEEeqnarray}{rCl}
\mathop{\hbox{max}}\nolimits_{\bV_t}\tr\{\bV_t'\bS_b\bV_t\},\quad s.t.\,\, \bV_t'\bS_t\bV_t=\bI_q,\label{eqn:lda.obj.t}
\end{IEEEeqnarray}
and then the columns of $\bV_t$ are $\bS_t$-orthogonal due to the constraint $\bV_t'\bS_t\bV_t=\bI_q$.
Note that the above two problems have close relationship as detailed in \refp{prop:prop.eqv}, e.g., $\bV_w$ can be obtained via $\bV_t$ by \refe{eqn:rlda.Vw}.

\subsection{Regularized LDA (RLDA)}\label{sec:rlda}
For high dimensional data, where the data dimension $d$ is greater than the sample size $n$, $\bS_t$ or $\bS_w$ is singular and hence LDA can not be performed. In addition, for the data whose variables are highly correlated even if $d$ is less than $n$, $\bS_t$ or $\bS_w$ could be approximately singular and the performance of LDA could be degenerated greatly. The regularized LDA (RLDA) proposed in \cite{friedman-rda} is a popular method to tackle these two problems simultaneously. The corresponding optimization problem to \refe{eqn:lda.obj.t} in RLDA is 
\begin{IEEEeqnarray}{rCl}
\mathop{\hbox{max}}\nolimits_{\bV_t}\tr\{\bV_t'\bS_b\bV_t\},\quad s.t.\,\, \bV_t'\bS_t^r\bV_t=\bI_q.\label{eqn:rlda.obj.t}
\end{IEEEeqnarray}
where $\bS_t^r=(1-r)\bS_t+r\hat{\sigma}^2\bI_q$, $\hat{\sigma}^2=\tr(\bS_t)/d$ and the regularization parameter $r\in(0, 1]$. Clearly, $\bS_t^r$ is shrunk to a scalar covariance. The parameter $r$ is usually determined via cross validation. 

By the method of Lagrange Multipliers, the closed form solution $\bV_t$ can be obtained by solving the generalized eigenvalue problem
\begin{equation}
{\bS_t^r}^{-1}\bS_b\bV_t=\bV_t\bLmd_t,\label{eqn:rlda.sol1}
\end{equation}
under the constraint $\bV_t'\bS_t^r\bV_t=\bI_q$. Similarly, the corresponding optimization problem to \refe{eqn:lda.obj.w}
\begin{IEEEeqnarray}{rCl}
\mathop{\hbox{max}}\nolimits_{\bV_w}\tr\{\bV_w'\bS_b\bV_w\},\quad s.t.\,\, \bV_w'\bS_w^r\bV_w=\bI_q.\label{eqn:rlda.obj.w}
\end{IEEEeqnarray}
where $\bS_w^r=(1-r)\bS_w+r\hat{\sigma}^2\bI_q$. In addition, we have 
\begin{equation}
\bS_t^r=\bS_w^r+(1-r)\bS_b.\label{eqn:rlda.StSw}
\end{equation}
Next we analyze the main time complexity of RLDA implemented by \refe{eqn:rlda.sol1}. The formation of $\bS_t^r$ costs $O(d^2n)$ and its inverse takes $O(d^3)$. The eigenproblem \refe{eqn:rlda.sol1} also costs $O(d^3)$. Thus the total cost is $O(d^2n)+O(d^3)$. When $d\gg n$, this cost could be very heavy. In \refs{sec:rlda.eff}, we review a computationally efficient implementation of RLDA.



\subsection{The relationship between the two problems}
\begin{prop}\label{prop:prop.eqv}
	Let $\bV_t$ be the solution of the RLDA problem \refe{eqn:rlda.obj.t} and $\bLmd_t=\bV_t'\bS_b\bV_t$. Then the solution $\bV_w$ of RLDA problem \refe{eqn:rlda.obj.w} and $\bLmd_w=\bV_w'\bS_b\bV_w$ are given by 
	\begin{IEEEeqnarray}{rCl}
		\bV_w&=&\bV_t(\bI_q-(1-r)\bLmd_t)^{-1/2},\label{eqn:rlda.Vw}\\
		\bLmd_w&=&\bLmd_t(\bI_q-(1-r)\bLmd_t)^{-1}.\nonumber 
	\end{IEEEeqnarray}
\end{prop}
\begin{proof}
	Substituting \refe{eqn:rlda.StSw} into $\bV_t'\bS_t^r\bV_t$ and noting $\bV_t'\bS_b\bV_t=\bLmd_t$, we have $\bV_t'\bS_w^r\bV_t=\bI_q-(1-r)\bLmd_t$. Define $\bV_w$ as in \refe{eqn:rlda.Vw}, we obtain $\bV_w'\bS_w^r\bV_w=\bI_q$. Furthermore, using \refe{eqn:rlda.Vw}, we obtain $\bLmd_w=\bV_w'\bS_b\bV_w=\bLmd_t(\bI_q-(1-r)\bLmd_t)^{-1}$. This completes the proof.
\end{proof}
\refp{prop:prop.eqv} shows the two problems are equivalent in terms of Fisher criterion $\cF$. However, the solutions $\bV_w$ and $\bV_t$ are different due to the diagonal transformation $(\bI_q-(1-r)\bLmd_t)^{-1/2}$ in \refe{eqn:rlda.Vw}. The subsequent classifiers such as nearest neighbors using $\bV_w$ and $\bV_t$ may yield different classification performances. 

\subsection{Efficient implementation of RLDA}\label{sec:rlda.eff}
Zhang \etal \cite{zhzhang-rlda} present an efficient implementation of RLDA problem \refe{eqn:rlda.obj.t}, which is briefly reviewed below. 

Let $\mathbf{1}_d$ stand for the $d\times 1$ vector of ones, and $\bE=(e_{ij})$ be a $n\times c$ indicator matrix with $e_{ij}=1$ if $\bx_i$ belonging to class $j$ and $e_{ij}=0$ otherwise. The $n\times n$ centering matrix $\bH=\bI_n-\frac1n\mathbf{1}_n\mathbf{1}'_n$.

Moreover, denote $\bX=[\bx_1,\dots,\bx_n]\in\R^{d\times n}$, $\bM=[\bm_1,\dots,\bm_c]\in\R^{d\times c}$, $\bPi=\diag{(n_1,\dots,n_c)}\in\R^{c\times c}$, $\bPi^{1/2}=\diag{(\sqrt{n}_1,\dots,\sqrt{n}_c)}$, $\bpi=(n_1,\dots,n_c)'\in\R^{c\times 1}$, $\sqrt{\bpi}=(\sqrt{n}_1,\dots,\sqrt{n}_c)'$. 

Without loss of generality, we assume that the data has been centered, i.e., $\bX=\bX\bH$. With the above notations, $\bS_b$ in \refe{eqn:lda.Sb} and $\bS_t$ in \refe{eqn:lda.St} can be rewritten as
\begin{IEEEeqnarray}{rCl}
	\bS_t&=&\bX\bX',\label{eqn:rlda.St}\\
	\bS_b&=&\bF_b\bF_b',\label{eqn:rlda.Sb}
\end{IEEEeqnarray}
where $\bF_b=\bX\bE\bPi^{-1/2}$. Substituting \refe{eqn:rlda.St} and \refe{eqn:rlda.Sb} into \refe{eqn:rlda.sol1}, we obtain
\begin{equation*}
	\bG\bF_b'\bV=\bV\bLmd,\label{eqn:rlda.sol.re}
\end{equation*}
where 
\begin{equation}
	\bG=((1-r)\bX\bX'+r\hat{\sigma}^2\bI_d)^{-1}\bF_b.\label{eqn:rlda.G.1}
\end{equation}
Since $((1-r)\bX\bX'+r\hat{\sigma}^2\bI_d)^{-1}\bX=\bX((1-r)\bX'\bX+r\hat{\sigma}^2\bI_n)^{-1}$, we also have 
\begin{equation}
	\bG=\bX((1-r)\bX'\bX+r\hat{\sigma}^2\bI_n)^{-1}\bE\bPi^{-1/2}.\label{eqn:rlda.G.2}
\end{equation}
When $n<d$, it is cheaper to compute $\bG$ by \refe{eqn:rlda.G.2}. Let 
\begin{IEEEeqnarray}{rCl}
	\bR&=&\bF_b'\bG=\bF_b'((1-r)\bX\bX'+r\hat{\sigma}^2\bI_d)^{-1}\bF_b.\label{eqn:rlda.R}
\end{IEEEeqnarray}
Since the $c\times c$ matrix $\bR$ and $d\times d$ matrix $\bG\bF_b'$ have the same nonzero eigenvalues, if $(\bLmd,\bV_R)$ is the eigenpair of $\bR$, then we have that $(\bLmd,\bG\bV_R)$ is the eigenpair of $\bG\bF_b'$. The eigenpair of $\bR$ can be obtained by SVD, which is also equivalent to eigenvalue decomposition (EVD) since $\bR$ is a nonnegative definite matrix. For clarity, the algorithm for RLDA is summarized in \refa{alg:rlda}.
\begin{algorithm}[htbp]
	\caption{Efficient implementation of RLDA.}
	\label{alg:rlda}
	\begin{algorithmic}[1]
		\REQUIRE ($\bX,\bE,\bPi,r$).
		\STATE Compute $\bG$ by \refe{eqn:rlda.G.1} or \refe{eqn:rlda.G.2} and $\bR$ by \refe{eqn:rlda.R}.
		\STATE Perform the condensed SVD of $\bR$ as $\bR=\bV_R\bLmd\bV'_R$.		
		\ENSURE $\bV_w=\bG\bV_R[\bLmd(\bI_q-(1-r)\bLmd)]^{-1/2}$.
	\end{algorithmic}
\end{algorithm}

The complexity analysis of \refa{alg:rlda} is given below. Calculating $\bG$ costs $O(nda)+O(a^3)$, where $a=\min(n,d)$. Computing $\bR$ takes $O(dc^2)$, its eigen-decomposition $O(c^3)$ and $\bG\bV$ $O(dc^2)$. Thus the total cost is $O(nda)+O(dc^2)$. It can be seen that \refa{alg:rlda} is more efficient than the direct implementation in \refs{sec:rlda}, particularly when $d\gg n$. In \refs{sec:rblda}, we will extend this algorithm to our proposed regularized bilinear LDA.

Another efficient implementation of RLDA based on SVD is also given in \cite{jpye-rlda}. However, it has been shown in \cite{zhzhang-rlda} that \refa{alg:rlda} based on EVD is more efficient than that based on SVD in \cite{jpye-rlda}.

\subsection{Review of bilinear LDA (BLDA)}\label{sec:blda}
Instead of a linear transformation, bilinear LDA (BLDA) seeks for a bilinear transformation $\bY=\bV_1'\bX\bV_2$. The idea originates from bilinear PCA \cite{zhangdq-2dpca}. To find the column and row transformations, \cite{Noushath-2dlda,inoue-2dlda,zhao2012-slda} present a separate solution. A brief review of the implementation in \cite{inoue-2dlda} is given below. 

Given a set of $d_1\times d_2$-dimensional matrix-valued data $\{\bX_{1}, \bX_{2}, \dots, \bX_n\}$ consisting of $c$ classes, let  $\bW=\frac{1}{n}\sum_i\bX_i$ and $\bW_k=\frac{1}{n_k}\sum_{i\in\cL_k}\bX_i$ be the global sample mean and the sample mean of class $k$, respectively. The column-column and row-row within-class and between-class scatter matrices are defined as
\begin{IEEEeqnarray}{rCl}
	\bS_{1w}&=&\frac1{Nd_2}\sum\nolimits_k\sum\nolimits_{i\in\cL_k}(\bX_i-\bW_k)(\bX_i-\bW_k)',\nonumber\\
	\bS_{2w}&=&\frac1{Nd_1}\sum\nolimits_k\sum\nolimits_{i\in\cL_k}(\bX_i-\bW_k)'(\bX_i-\bW_k),\nonumber\\ 
	\bS_{1b}&=&\frac1{Nd_2}\sum\nolimits_kn_k(\bW_k-\bW)(\bW_k-\bW)', \label{eqn:blda.S1b} \\ \bS_{2b}&=&\frac1{Nd_1}\sum\nolimits_kn_k(\bW_k-\bW)'(\bW_k-\bW).\label{eqn:blda.S2b}
\end{IEEEeqnarray}
It can be easily verified that the ranks of these matrices satisfy $\rank(\bS_{1w})\leq\min(d_1,d_2(n-c))$, $\rank(\bS_{2w})=\min(d_2,d_1(n-c))$, $\rank(\bS_{1b})\leq\min{(d_1,d_2(c-1))}$ and $\rank(\bS_{2b})\leq\min$ $(d_1(c-1),d_2)$. 

The Fisher criterion for BLDA can be formulated as two separate Fisher sub-criterions \cite{zhao2012-slda}
\begin{IEEEeqnarray*}{rCl}
	\mathop{\hbox{arg\,max}}\nolimits_{\bV_{lw}}\hbox{tr}\left\{(\bV_{lw}'\bS_{lw}\bV_{lw})^{-1}(\bV_{lw}'\bS_{lb}\bV_{lw})\right\}, 
\end{IEEEeqnarray*}
$l=1,2$, each of which is a similar optimization problem to that in LDA, with $l=1$ devoted to column direction and $l=2$ to row direction. The solution is obtained by solving the two problems separately
\begin{IEEEeqnarray}{rCl}
	\hskip-1em\mathop{\hbox{max}}\nolimits_{\bV_{lw}}\tr\left\{\bV_{lw}'\bS_{lb}\bV_{lw}\right\},\, s.t.\, \bV_{lw}'\bS_{lw}\bV_{lw}=\bI,\label{eqn:blda.obj.w}
\end{IEEEeqnarray}
where $l=1,2$. Assume that $\bS_{lw}$ is invertible and let the EVD of $\bS_{lw}$ be 
\begin{equation*}
	\bS_{lw}=\bU_{lw}\bGam_{lw}\bU'_{lw}.\label{eqn:blda.Sw.SVD}
\end{equation*}
Denote $\btV_{lw}=\bU_{lw}\bGam_{lw}^{1/2}\bV_{lw}$. Substitute $\bV_{lw}=\bU_{lw}\bGam_{lw}^{-1/2}\btV_{lw}$ into \refe{eqn:blda.obj.w}, yielding
\begin{IEEEeqnarray}{rCl}
	\mathop{\hbox{max}}\nolimits_{\btV_{lw}}\tr\left\{\btV_{lw}'\bR_l\btV_{lw}\right\},\,\, s.t.\,\,\btV_{lw}'\btV_{lw}=\bI, \label{eqn:blda.obj1.w}
\end{IEEEeqnarray}
where $\bR_l=\bGam_{lw}^{-1/2}\bU'_{lw}\bS_{lb}\bU_{lw}\bGam_{lw}^{-1/2}$. Denote the solution to problem \refe{eqn:blda.obj1.w} be $\bV_{lR}$, the solution to problem \refe{eqn:blda.obj.w} is then given by $\bV_{lw}=\bU_{lw}\bGam_{lw}^{-1/2}\bV_{lR}$.

Next we analyze the main time complexity of this procedure in column direction since the analysis in row direction is similar. The formation of the $d_1\times d_1$ matrix $\bS_{1w}$ costs $O(d_1^2d_2n)$ and its eigen-decomposition takes $O(d_1^3)$. Computing $\bR_1$ and its eigen-decomposition cost $O(d_1^3)$. Hence the total cost for $\bV_{1w}$ is $O(d_1^2d_2n)+O(d_1^3)$. It can be seen that the cost could be very heavy when $d_1$ is much greater than $d_2$.

Another shortcoming of this procedure is that it requires both of $\bS_{1w}$ and $\bS_{2w}$ to be invertible. Since $\rank(\bS_{1w})\leq\min(d_1,d_2(n-c))$ and $\rank(\bS_{2w})=\min(d_2,d_1(n-c))$, for data $d_1>d_2(n-c)$ or $d_2>d_1(n-c)$, $\bS_{1w}$ or $\bS_{2w}$ is singular and thus this procedure cannot be performed. Motivated by the success that RLDA improves LDA when the within-class or total scatter matrix is singular or approximately singular, in \refs{sec:rblda} we follow the idea of RLDA and propose regularized bilinear LDA (RBLDA) to overcome the two problems suffered by BLDA.

\section{Regularized Bilinear LDA (RBLDA)}\label{sec:rblda}
In this section, we extend the idea of RLDA to BLDA and propose regularized bilinear LDA (RBLDA) for MTS data classification. In \refs{sec:rblda.fml}, we formulate the optimization problem for our proposed RBLDA. In \refs{sec:rblda.alg}, we present an efficient implementation, based on which, in \refs{sec:rblda.ms} we propose an efficient model selection algorithm for RBLDA. 
\subsection{Regularized BLDA (RBLDA)}\label{sec:rblda.fml}
Similar to those in RLDA, we define column-column and row-row total scatter matrices as
\begin{IEEEeqnarray}{rCl}
	\bS_{1t}&=&\frac1{Nd_2}\sum\nolimits_i(\bX_i-\bW)(\bX_i-\bW)',\label{eqn:blda.S1t}\\ \bS_{2t}&=&\frac1{Nd_1}\sum\nolimits_i(\bX_i-\bW)'(\bX_i-\bW).\label{eqn:blda.S2t}
\end{IEEEeqnarray}
We solve the following two eigenproblems,
\begin{equation}
	{\bS_{lt}^{r_l}}^{-1}\bS_{lb}\bV_l=\bV_l\bLmd_l,\quad l=1,2,\label{eqn:rblda.sol1}
\end{equation}
under the constraints that $\bV_l'\bS_{lt}\bV_l=\bI_{q_l}$, where  
\begin{equation}
	\bS_{lt}^{r_l}=(1-r_l)\bS_{lt}+r_l\hat{\sigma}_l^2\bI_{d_l},\label{rblda.reguS1}
\end{equation}
$\hat{\sigma}_l^2=\tr(\bS_{lt})/d_l$, and the regularization parameters $r_1,r_2\in(0, 1]$. Note that the solutions $\bV_{lw}$ to the optimization problems \refe{eqn:rblda.sol1} under the constraints that $\bV_{lw}'\bS_{lw}\bV_{lw}=\bI$ can be easily obtained by \refe{eqn:rlda.Vw}.

\subsection{Efficient implementations}\label{sec:rblda.alg}
In this subsection, we develop efficient implementations for RBLDA with $(r_1, r_2)$ given. 
\subsubsection{First efficient implementation of RBLDA with $(r_1, r_2)$ given }\label{sec:imp.col}
We extend the efficient algorithm of RLDA in \refs{sec:rlda.eff} to RBLDA with regularization parameter $(r_1, r_2)$ given. We first consider the implementation in column direction with $r_1$ given. Let 
\begin{IEEEeqnarray*}{rCl}
	\bX_{(1)}&=&[\bX_1,\dots,\bX_n]\,\, (d_1\times d_2n), \\
	\bW_{(1)}&=&[\bW_1,\dots,\bW_c]\,\, (d_1\times d_2c), \\
	\bPi_1&=&\bdiag{(n_1\bI_{d_2},\dots,n_c\bI_{d_2})}\,\, (d_2c\times d_2c),\\
	\bPi_1^{1/2}&=&\bdiag{(\sqrt{n}_1\bI_{d_2},\dots,\sqrt{n}_c\bI_{d_2})}\,\, (d_2c\times d_2c), \\
	\bpi_1&=&(n_1\bI_{d_2},\dots,n_c\bI_{d_2})'\,\, (d_2c\times d_2),\\
	\sqrt{\bpi_1}&=&(\sqrt{n}_1\bI_{d_2},\dots,\sqrt{n}_c\bI_{d_2})'\,\, (d_2c\times d_2), \\
	\mathbb{I}_{1m}&=&(\bI_{d_2},\dots,\bI_{d_2})'\,\, (d_2m\times d_2),\\
	\bH_1&=&\bI_{d_2n}-\frac1n\mathbb{I}_{1n}\mathbb{I}'_{1n}\,\, (d_2n\times d_2n), \\ 
	\bH_{1\pi}&=&\bI_{d_2c}-\frac1n\sqrt{\bpi_1}\sqrt{\bpi_1}'\,\, (d_2c\times d_2c), 
\end{IEEEeqnarray*}
$\bbE_1=(\bE_{ij})\,\, (d_2n\times d_2c)$ be a $n\times c$ partitioned matrix with the $d_2\times d_2$ $ij$-th submatrices $\bE_{ij}=\bI_{d_2}$ if $\bX_i$ belonging to class $j$ and $\bE_{ij}=\bo_{d_2}$ otherwise. 

Without loss of generality, we assume that the data has been centered, i.e., $\bX_{(1)}=\bX_{(1)}\bH_1$. With these representations, we have
$\mathbb{I}'_{1n}\bbE_1=\mathbb{I}'_{1c}\bPi_1=\bpi_1'$, $\bbE_1\mathbb{I}_{1c}=\mathbb{I}_{1n}$, $\mathbb{I}'_{1c}\bpi_1=n\bI_{d_2}$, $\bbE'_1\bbE_1=\bPi_1$, $\bPi_1^{-1}\bpi_1=\mathbb{I}_{1c}$, and
\begin{IEEEeqnarray*}{rCl}
	\bW_{(1)}&=&\bX_{(1)}\bbE_1\bPi_1^{-1},\\
	\bbE_1\bPi_1^{-1/2}\bH_{1\bpi}&=&\bbE_1\bPi_1^{-1/2}-\frac1n\bbI_{1n}\sqrt{\bpi_1}'=\bH_1\bbE_1\bPi_1^{-1/2},
\end{IEEEeqnarray*}
Noting that $\bH_1=\bH_1^2$, $\bS_{1t}$ in \refe{eqn:blda.S1t} can be rewritten as
\begin{equation}
	\bS_{1t}=\bX_{(1)}\bX'_{(1)},\label{eqn:rblda.S1t}
\end{equation}
and $\bS_{1b}$ in \refe{eqn:blda.S1b} as
\begin{IEEEeqnarray}{rCl}
	\bS_{1b}&=&\bW_{(1)}\left[\bPi_1-\frac1n\bpi_1\bpi'_1\right]\bW'_{(1)}=\bW_{(1)}\cdot\nonumber\\
	&&\hskip-1.5em\left[\bPi_1^{1/2}-\frac1n\bpi_1\sqrt{\bpi}'_1\right]\left[\bPi_1^{1/2}-\frac1n\sqrt{\bpi_1}\bpi'_1\right]\bW'_{(1)},\nonumber\\
	&=&\bX_{(1)}\bbE_1\bPi_1^{-1}\bPi_1^{1/2}\bH_{1\pi}\bH_{1\pi}\bPi_1^{1/2}\bPi_1^{-1}\bbE_1\bX'_{(1)},\nonumber\\
	&=&\bX_{(1)}\bbE_1\bPi_1^{-1}\bbE'_1\bX'_{(1)}\nonumber\\
	&=&\bF_{1b}\bF_{1b}',\label{eqn:rblda.S1b}
\end{IEEEeqnarray}
where $\bF_{1b}=\bX_{(1)}\bbE_1\bPi_1^{-1/2}$. In addition, we define 
\begin{IEEEeqnarray}{rCl}
	\bT_1&=&\bX'_{(1)}\bX_{(1)},\nonumber\\
	\bT_1^{r_1}&=&(1-r_1)\bX'_{(1)}\bX_{(1)}+r_1\hat{\sigma}_1^2\bI_{d_2n}\label{rblda.reguT1}
\end{IEEEeqnarray}
Let the condensed SVD of $\bS_{1t}(d_1\times d_1)$ in \refe{eqn:rblda.S1t} and $\bT_1(d_2n\times d_2n)$ be 
\begin{IEEEeqnarray}{rCl}
	\bS_{1t}&=&\bX_{(1)}\bX'_{(1)}=\bU_{1t}\bGam_{1t}\bU'_{1t},\label{eqn:rblda.S1t.SVD}\\
	\bT_1&=&\bX'_{(1)}\bX_{(1)}=\bU_{1T}\bGam_{1t}\bU'_{1T},\label{eqn:rblda.T1.SVD}
\end{IEEEeqnarray}
where $t_1=\rank(\bGam_{1t})$. When $d_1>d_2n$, the condensed SVD in \refe{eqn:rblda.S1t.SVD} can be obtained via \refe{eqn:rblda.T1.SVD}. Thus in this case the complexity of the condensed SVD of $\bS_{1t}$ is $O(d_1(d_2n)^2)$.

Accordingly, the EVD of $\bS_{1t}^{r_1}$ in \refe{rblda.reguS1} and $\bT_1^{r_1}$ \refe{rblda.reguT1} are given by
\begin{IEEEeqnarray}{rCl}
	\bS_{1t}^{r_1}&=&\bU_{1t}\bGam_{1t}^{r_1}\bU'_{1t}+r_1\hat{\sigma}_1^2\bU_{1t}^\perp{\bU_{1t}^\perp}',\label{eqn:rblda.S1tr1.SVD}\\
	\bT_1^{r_1}&=&\bU_{1T}\bGam_{1T}^{r_1}\bU'_{1T}+r_1\hat{\sigma}_1^2\bU_{1T}^\perp{\bU_{1T}^\perp}',\nonumber
\end{IEEEeqnarray}
where $\bGam_{1t}^{r_1}=(1-r_1)\bGam_{1t}+r_1\hat{\sigma}_1^2\bI_{d_1}$ $\bGam_{1T}^{r_1}=(1-r_1)\bGam_{1T}+r_1\hat{\sigma}_1^2\bI_{d_2n}$, and $\bU_{it}^\perp$ is the orthogonal complement of $\bU_{it}, i=1,2$.

Next, we solve the eigenproblem in \refe{eqn:rblda.sol1} in two cases: $d_2c<d_1$ and $d_2c\geq d_1$, under which we will see that the problem can be implemented with different computational costs.

(i) When $d_2c\geq d_1$, with \refe{eqn:rblda.S1tr1.SVD}, the eigenproblem in \refe{eqn:rblda.sol1} is
\begin{equation*}
	\bU_{1t}{\bGam_{1t}^{r_1}}^{-1/2}\bG_{11}\bV_1=\bV_1\bLmd_1,\label{eqn:rblda.sol1.regu.1}
\end{equation*}
where 
\begin{equation*}
	\bG_{11}={\bGam_{1t}^{r_1}}^{-1/2}\bU'_{1t}\bS_{1b}.\label{eqn:rblda.G11}
\end{equation*}
Let 
\begin{IEEEeqnarray}{rCl}
	\bR_{11}&=&\bG_{11}\bU_{1t}{\bGam_{1t}^{r_1}}^{-1/2}\nonumber\\
	&=&{\bGam_{1t}^{r_1}}^{-1/2}\bU'_{1t}\bS_{1b}\bU_{1t}{\bGam_{1t}^{r_1}}^{-1/2}.\label{eqn:rblda.R11}
\end{IEEEeqnarray}
Since the $d_1\times d_1$ matrices $\bR_{11}$ and $\bU_{1t}{\bGam_{1t}^{r_1}}^{-1/2}\bG_{11}$ have the same nonzero eigenvalues, if $(\bLmd_1,\bV_{1R})$ is the eigenpair of $\bR_{11}$ and $\bV_1=\bU_{1t}{\bGam_{1t}^{r_1}}^{-1/2}\bV_{1R}$, we have $(\bLmd_1,\bV_1)$ is the eigenpair of $\bU_{1t}{\bGam_{1t}^{r_1}}^{-1/2}\bG_{11}$. 

(ii) When $d_2c<d_1$, using \refe{eqn:rblda.S1b}, the eigenproblem in \refe{eqn:rblda.sol1} can be represented by
\begin{equation*}
	\bG_{12}\bF_{1b}'\bV_1=\bV_1\bLmd_1,\label{eqn:rblda.sol1.regu.2}
\end{equation*}
where 
\begin{equation}
	\bG_{12}={\bS_{1t}^{r_1}}^{-1}\bF_{1b}.\,\,(d_1\times d_2c)\label{eqn:rblda.G12.1}
\end{equation}
Since ${\bS_{1t}^{r_1}}^{-1}\bX_{(1)}=\bX_{(1)}{\bT_1^{r_1}}^{-1}$, we also have 
\begin{equation}
	\bG_{12}=\bX_{(1)}{\bT_1^{r_1}}^{-1}\bbE_1\bPi_1^{-1/2}.\label{eqn:rblda.G12.2}
\end{equation}
When $d_2n<d_1$, it is cheaper to compute $\bG_{12}$ by \refe{eqn:rblda.G12.2}. Let 
\begin{IEEEeqnarray}{rCl}
	\bR_{12}&=&\bPi_1^{-1/2}\bbE_1'\bX'_{(1)}\bG_{12}\quad(d_2c\times d_2c)\label{eqn:rblda.R12}\\
	&=&\bF_{1b}'{\bS_{1t}^{r_1}}^{-1}\bF_{1b}.\nonumber
\end{IEEEeqnarray}
The $d_2c\times d_2c$ matrix $\bR_{12}$ and $d_1\times d_1$ matrix $\bG_{12}\bF_{1b}'$ have the same nonzero eigenvalues, if $(\bLmd_1,\bV_{1R})$ is the eigenpair of $\bR_{12}$ and $\bV_1=\bG_{12}\bV_{1R}$, $(\bLmd_1,\bV_1)$ is the eigenpair of $\bG_{12}\bF_{1b}'$. 

When $d_2c\geq d_1$, the cost is the same as that of BLDA in \refs{sec:blda}, i.e.,  $O(d_1^2d_2n)+O(d_1^3)$. When $d_2c<d_1$, calculating $\bG_{12}$ via \refe{eqn:rblda.G12.1} or \refe{eqn:rblda.G12.2} costs $O(d_1d_2na_1)+O(d_1d_2ca_1)$, where $a_1=\min(d_1,d_2n)$. Computing $\bR_{12}$ and $\bG_{12}\bV_1$ takes $O(d_1(d_2c)^2)$. Feature extraction for $\bX$ (i.e., computation of $\bV_1'\bX_{(1)}$) takes $O(d_1d_2nd_2c)$. Thus the total cost is $O(d_1d_2na_1+d_1d_2nd_2c)$. Compared with that of BLDA, the cost with this implementation is substantially reduced. However, if a total of $m_1$ candidates of $r_1$ are considered, then the total cost of running this procedure would be $O(m_1(d_1d_2na_1+d_1d_2nd_2c))$.

\subsubsection{Second efficient implementation of RBLDA with $(r_1, r_2)$ given}\label{sec:rblda.eff.imp}
In this subsection, we shall develop an efficient implementation for RBLDA with $(r_1, r_2)$ given, which is a variant of that in \refs{sec:imp.col}. The advantage is that this implementation allows that the cross validation for RBLDA on all candidates of $(r_1, r_2)$ can be performed efficiently. We first consider the implementation in column direction with $r_1$ given. Let 
\begin{IEEEeqnarray}{rCl}
	\bX_{(1u)}=\bU_{1t}'\bX_{(1)},\,\,(t_1\times d_2n).\label{eqn:X1proj}
\end{IEEEeqnarray}
where $t_1=\rank(\bX_{(1)})$. By \refe{eqn:rblda.S1tr1.SVD}, the regularized total scatter matrix in $\bX_{(1u)}$-space is   
\begin{IEEEeqnarray}{rCl}
	\bS_{1tu}^{r_1}&=&\bGam_{1t}^{r_1}.\label{eqn:rblda.S1tr1u.SVD}
\end{IEEEeqnarray}

Define
\begin{IEEEeqnarray}{rCl}
	\bF_{1bu}&=&\bX_{(1u)}\bbE_1\bPi_1^{-1/2}.\,\,(t_1\times d_2c)\label{eqn:rblda.F1bu}
\end{IEEEeqnarray}
(i) When $d_2c\geq d_1$, with \refe{eqn:rblda.F1bu}, $\bR_{11}$ \refe{eqn:rblda.R11} in $\bX_{(1u)}$-space remains unchanged, i.e.,
\begin{IEEEeqnarray}{rCl}
	\bR_{11}={\bGam_{1t}^{r_1}}^{-1/2}\bF_{1bu}\bF_{1bu}'{\bGam_{1t}^{r_1}}^{-1/2},\label{eqn:rblda.R11u}
\end{IEEEeqnarray}
and the $\bV$-solution in $\bX_{(1u)}$-space is $\bV_{1u}={\bGam_{1t}^{r_1}}^{-1/2}\bV_{1R} (t_1\times d_2c)$.\\
(ii) When $d_2c< d_1$, by \refe{eqn:rblda.S1tr1.SVD}, $\bG_{12}$ and $\bR_{12}$ in $\bX_{(1u)}$-space (\refe{eqn:rblda.G12.1} and \refe{eqn:rblda.R12}) are
\begin{IEEEeqnarray}{rCl}
	\bG_{12u}&=&{\bGam_{1t}^{r_1}}^{-1}\bF_{1bu},\,\,(t_1\times d_2c)\label{eqn:rblda.G12.1u}\\
	\bR_{12}&=&\bF_{1bu}'\bG_{12u}=\bF_{1bu}'{\bGam_{1t}^{r_1}}^{-1}\bF_{1bu},\label{eqn:rblda.R12u}
\end{IEEEeqnarray}
and the $\bV$-solution in $\bX_{(1u)}$-space is $\bV_{1u}=\bG_{12u}\bV_{1R} (t_1\times d_2c)$. It will be seen from \refs{sec:rblda.ms} that when $d_1$ is much greater than $d_2$, the cost of this implementation for $r_1$ is $O(d_1d_2na_1+d_2nd_2ca_1)$. Since $a_1\leq d_1$, the cost could be slightly reduced compared with that in \refs{sec:imp.col}.


The $\bV$-solution in the original $\bX_{(1)}$-space can be obtained by $\bV_1=\bU_{1t}\bV_{1u}$ under both cases. However, this is not necessary if we use the nearest neighbor classifier since the distance computations based on $\bV_{1u}$ and $\bV_1$ are the same.  

The computation in row direction with $r_2$ given is similar to that in column direction. We denote 
\begin{IEEEeqnarray*}{rCl}
	\bX_{(2)}&=&[\bX'_1,\dots,\bX'_n]\,\, (d_2\times d_1n), \\
	\bW_{(2)}&=&[\bW'_1,\dots,\bW'_c]\,\, (d_2\times d_1c),\\ \bPi_2&=&\bdiag{(n_1\bI_{d_1},\dots,n_c\bI_{d_1})}\,\, (d_1c\times d_1c),\\ \
	\bPi_2^{1/2}&=&\bdiag{(\sqrt{n}_1\bI_{d_1},\dots,\sqrt{n}_c\bI_{d_1})}\,\, (d_1c\times d_1c), \\ 
	\bpi_2&=&(n_1\bI_{d_1},\dots,n_c\bI_{d_1})'\,\, (d_1c\times d_1),\\ 
	\sqrt{\bpi_2}&=&(\sqrt{n}_1\bI_{d_1},\dots,\sqrt{n}_c\bI_{d_1})'\,\, (d_1c\times d_1), \\
	\mathbb{I}_{2m}&=&(\bI_{d_1},\dots,\bI_{d_1})'\,\, (d_1m\times d_1),\\
	\bH_2&=&\bI_{d_1n}-\frac1n\mathbb{I}_{2n}\mathbb{I}'_{2n}\,\, (d_1n\times d_1n), \\
	\bH_{2\pi}&=&\bI_{d_1c}-\frac1n\sqrt{\bpi_2}\sqrt{\bpi_2}'\,\, (d_1c\times d_1c),
\end{IEEEeqnarray*}
$\bbE_2=(\bE_{ij})\,\, (d_1n\times d_1c)$ be a $n\times c$ partitioned matrix with the $d_1\times d_1$ $ij$-th submatrices $\bE_{ij}=\bI_{d_1}$ if $\bX_i$ belonging to class $j$ and $\bE_{ij}=\bo_{d_1}$ otherwise.

Assuming that $\bX_{(2)}=\bX_{(2)}\bH_2$, $\bS_{2b}$ in \refe{eqn:blda.S2b} is rewritten as
\begin{IEEEeqnarray}{rCl}
	\bS_{2b}=\bX_{(2)}\bbE_2\bPi_2^{-1}\bbE'_2\bX'_{(2)},\label{eqn:rblda.S2b}
\end{IEEEeqnarray}
The condensed SVD of $\bS_{2t}$ in \refe{eqn:blda.S2t} is 
\begin{IEEEeqnarray}{rCl}
	\bS_{2t}=\bX_{(2)}\bX'_{(2)}=\bU_{2t}\bGam_{2t}\bU'_{2t},\label{eqn:rblda.S2t.SVD}
\end{IEEEeqnarray}
where $t_2=\rank(\bGam_{2t})$. Accordingly, the EVD of $\bS_{2t}^{r_2}$ in \refe{rblda.reguS1} is given by 
\begin{IEEEeqnarray}{rCl}
	\bS_{2t}^{r_2}=\bU_{2t}\bGam_{2t}^{r_2}\bU'_{2t}+r_2\hat{\sigma}_2^2\bU_{2t}^\perp{\bU_{2t}^\perp}',\label{eqn:rblda.S2tr2.SVD}
\end{IEEEeqnarray}
where $\bGam_{2t}^{r_2}=(1-r_2)\bGam_{2t}+r_2\hat{\sigma}_2^2\bI_{d_1}$.
Let 
\begin{IEEEeqnarray}{rCl}
	\bX_{(2u)}=\bU_{2t}'\bX_{(2)}.\,\,(t_2\times d_1n)\label{eqn:X2proj}
\end{IEEEeqnarray}
By \refe{eqn:rblda.S2tr2.SVD}, the regularized total scatter matrices in $\bX_{(2u)}$-space is 
\begin{IEEEeqnarray*}{rCl}
	\bS_{2tu}^{r_1}=\bGam_{2t}^{r_2},\label{eqn:rblda.S2tr2u.SVD}
\end{IEEEeqnarray*}

Define
\begin{IEEEeqnarray}{rCl}
	\bF_{2bu}=\bX_{(2u)}\bbE_2\bPi_2^{-1/2}.\,\,(t_2\times d_1c)\label{eqn:rblda.F2bu}
\end{IEEEeqnarray}
(i) When $d_1c\geq d_2$, let 
\begin{IEEEeqnarray}{rCl}
	\bR_{21}={\bGam_{2t}^{r_2}}^{-1/2}\bF_{2bu}\bF_{2bu}'{\bGam_{2t}^{r_2}}^{-1/2}.\label{eqn:rblda.R21}
\end{IEEEeqnarray}
If $(\bLmd_2,\bV_{2R})$ is the eigenpair of $\bR_{21}$ and  $\bV_{2u}={\bGam_{2t}^{r_2}}^{-1/2}\bV_{2R}$, then $(\bLmd_2,\bV_{2u})$ is the solution in $\bX_{(2u)}$-space.\\ 
(ii) When $d_1c<d_2$, let 
\begin{IEEEeqnarray}{rCl}
	\bG_{22u}&=&{\bGam_{2t}^{r_2}}^{-1}\bF_{2bu}, \label{eqn:rblda.G22.1u}\\
	\bR_{22}&=&\bF_{2bu}'{\bGam_{2t}^{r_2}}^{-1}\bF_{2bu}.\label{eqn:rblda.R22u}
\end{IEEEeqnarray}
If $(\bLmd_2,\bV_{2R})$ is the eigenpair of $\bR_{22}$, and $\bV_{2u}=\bG_{22u}\bV_{2R}$ then $(\bLmd_2,\bV_{2u})$ is solution in $\bX_{(2u)}$-space. 

\subsection{Efficient model selection algorithm for RBLDA}\label{sec:rblda.ms}
Based on the implementation in \refs{sec:rblda.eff.imp}, in this subsection we propose an efficient model selection algorithm for RBLDA on all candidates of $(r_1, r_2)$.

Thanks to \refe{eqn:X1proj}, which is the same for $m_1$ candidates of $r_1$, and hence the cost of the RBLDA algorithm in \refs{sec:rblda.eff.imp} for $m_1$ candidates of $r_1$ can be reduced to $O(d_1d_2na_1+m_1(d_2nd_2ca_1))$ only. In other words, the cross validation for RBLDA can be performed efficiently. For clarity, the whole model selection algorithm of RBLDA is summarized in \refa{alg:rblda}.
\begin{algorithm}[htbp]
	\caption{Efficient RBLDA model selection algorithm.}
	\label{alg:rblda}
	\begin{algorithmic}[1]
		\REQUIRE ($\bX,\bbE_l,\bPi_l,\{r_{li}\}_{i=1}^{m_l}, l=1,2$).
		\FOR{$v=1:V$}
		\STATE Split $\bX$ as training $\bX_v$ and validation $\bX_v^\perp$; 
		\STATE Perform condensed SVD of $\bS_{1t}$ by \refe{eqn:rblda.S1t.SVD} on $\bX_v$. \label{alg:S1t.SVD}
		\STATE Compute $\bX_{(1u)}$ by \refe{eqn:X1proj}, $\bF_{1bu}$ by \refe{eqn:rblda.F1bu} on $\bX_v$. \label{alg:F1bu}
		\FOR{$i=1:m_1$}%
		\STATE Given $r_{1i}$, \COMMENT $i$-th candidate from $m_1$ choices of $r_1$.
		\STATE When $d_2c\geq d_1$, compute $\bR_{11}$ by \refe{eqn:rblda.R11u}, and \label{alg:R11}
		\STATE compute $\bG_{12u}$ by \refe{eqn:rblda.G12.1u} and $\bR_{12}$ by \refe{eqn:rblda.R12u} otherwise. \label{alg:R12}		
		\STATE Perform the condensed SVD of $\bR_{11}$ or $\bR_{12}$ as $\bV_{1R}\bLmd_1\bV'_{1R}$ and set $\bV_{1ui}={\bGam_{1t}^{r_1}}^{-1/2}\bV_{1R}$ or $\bV_{1ui}=\bG_{12u}\bV_{1R}[\bLmd_1(\bI-(1-r_1)\bLmd_1)]^{-\frac12}$.\label{alg:R1.SVD} 
		\ENDFOR	
		\STATE Perform condensed SVD of $\bS_{2t}$ by \refe{eqn:rblda.S2t.SVD} on $\bX_v$.\label{alg:S2t.SVD}
		\STATE Compute $\bX_{(2u)}$ by \refe{eqn:X2proj}, $\bF_{2bu}$ by \refe{eqn:rblda.F2bu} on $\bX_v$.\label{alg:F2bu}
		\FOR{$j=1:m_2$}
		\STATE Given $r_{2j}$, \COMMENT $j$-th candidate from $m_2$ choices of $r_2$.
		\STATE When $d_1c\geq d_2$, compute $\bR_{21}$ by \refe{eqn:rblda.R21}, and \label{alg:R21}
		\STATE compute $\bG_{22u}$ by \refe{eqn:rblda.G22.1u} and $\bR_{22}$ by \refe{eqn:rblda.R22u} otherwise.\label{alg:R22}			
		\STATE Perform the condensed SVD of $\bR_{21}$ or $\bR_{22}$ as $\bV_{2R}\bLmd_2\bV'_{2R}$ and set $\bV_{2uj}={\bGam_{2t}^{r_2}}^{-1/2}\bV_{2R}$ or $\bV_{2uj}=\bG_{22u}\bV_{2R}[\bLmd_2(\bI-(1-r_2)\bLmd_2)]^{-\frac12}$. \label{alg:R2.SVD}
		\ENDFOR	
		\FOR{$i=1:m_1$} 
		\FOR{$j=1:m_2$}
		\STATE $\bX_v\leftarrow\bV_{1ui}'\bX_v\bV_{2uj}$,\\ $\bX_v^\perp\leftarrow\bV_{1ui}'\bX_v^\perp\bV_{2uj}$.\label{alg:Xiproj}
		\STATE Run 1NN on $(\bX_v,\bX_v^\perp)$ and compute the error rate $\mbox{Err}(v,i,j)$. \label{alg:NN}
		\ENDFOR	
		\ENDFOR	
		\STATE $\mbox{Err}(i,j)=1/v\sum_{v=1}^V\mbox{Err}(v,i,j)$.
		\STATE $(i^*,j^*)\leftarrow \mbox{arg min}_{(i,j)}\mbox{Err}(i,j)$.
		\ENDFOR	
		\ENSURE $\{\bV_{1ui}\}_{i=1}^{m_1}$, $\{\bV_{2uj}\}_{j=1}^{m_2}$.
	\end{algorithmic}
\end{algorithm}

Below we analyze the complexity of the feature extraction stage of \refa{alg:rblda} (i.e., not involving the classification stage (Line~\ref{alg:NN})). Line~\ref{alg:S1t.SVD} takes $O(d_1d_2na_1)$ and line~\ref{alg:F1bu} costs $O(t_1d_1d_2n)$. When $d_2c\geq d_1$, line~\ref{alg:R11} takes $O(t_1^2d_2c)$ and line~\ref{alg:R1.SVD} $O(t_1^3)$, and line~\ref{alg:R12} and \ref{alg:R1.SVD} take $O(t_1(d_2c)^2)$ otherwise. Line~\ref{alg:Xiproj} takes $O(t_1d_2n\cdot\min(t_1,d_2c))$. The total cost for $m_1$ and $m_2$ candidates in column and row directions is
\begin{IEEEeqnarray*}{rCl}
	T(m_1,m_2)&=&O\left(d_1d_2na_1+m_1m_2(t_2d_2\min(t_1,d_2c)n)\right.\\
	&&\left.+\>m_1(t_1d_2c+t_1d_2n)\min(t_1,d_2c)\right).
\end{IEEEeqnarray*}

(i) When $d_2c>d_1$ and $d_1c>d_2$, i.e, $d_2\approx d_1$, the costs in column and row directions are about the same and thus  we only analyze the cost in column direction. In this case,  $d_2c>t_1$, $t_1\leq a_1$ and hence  
\begin{IEEEeqnarray*}{rCl}
	T(m_1,m_2)&=&O(d_1^2d_2n+m_1(d_1^2d_2n)+m_1m_2(d_2^2d_1n)).
\end{IEEEeqnarray*}
The ratio of $T(m_1,m_2)$ to $T(1,1)$ can be approximately expressed as 
$$\frac{T(m_1,m_2)}{T(1,1)}\approx\frac{m_1m_2}3.$$

(ii) When $d_1>d_2c$, in particular, $d_1$ is much higher than $d_2$, the cost mainly lies in column direction. For simplicity, we assume that $t_1>d_2c$ and obtain
\begin{IEEEeqnarray*}{rCl}
	T(m_1,m_2)&=&O(d_1d_2na_1+m_1(d_2nd_2ca_1)\\
	&&+\>m_1m_2(d_2^2d_2cn)).
\end{IEEEeqnarray*}
Note that for single $(r_1,r_2)$, the cost of $T(1,1)$ in this case can be expressed as  $O(d_1d_2na_1+d_2nd_2ca_1)$, which is mentioned in \refs{sec:rblda.eff.imp}.

The ratio of $T(m_1,m_2)$ to $T(1,1)$ can be approximately expressed as
$$\frac{T(m_1,m_2)}{T(1,1)}\approx1+m_1\frac{d_2c}{d_1}+m_1m_2\frac{d_2^2c}{d_1a_1}.$$

The complexity analysis shows that the proposed \refa{alg:rblda} is efficient when one data dimensionality is much higher than the other, i.e, $d_1\gg d_2n$ or $d_2\gg d_1n$.



\section{Experiments}\label{sec:expr}
In this section, we perform experiments on the following five publicly available real-world MTS data sets. 
\begin{enumerate}[-]
	\item AUSLAN dataset (AUS). AUS contains 2565 MTS observations of 95 signs (i.e. 95 classes). Each sign has 27 observations, each captured from a native AUSLAN speaker using 22 sensors (i.e. 22 variables) on the CyberGlove. In our experiments, we use a subset consisting of 675 observations of 25 signs, each 27 observations. The 25 signs are alive, all, boy, building, buy, cold, come, computer, cost, crazy, danger, deaf, different, girl, glove, go, God, joke, juice, man, where, which, yes, you and zero. The time length of MTS observation we use is 47 and hence the data dimension is $47\times 22$.
	
	\item ECG dataset. ECG comprises 200 MTS observations of two classes. Each observation is collected by two electrodes (i.e. 2 variables) during one heartbeat and labeled as normal or abnormal. Abnormal heartbeats signal a cardiac pathology known as supraventricular premature beat. The normal and abnormal classes have 133 and 67 observations, respectively. The time length used is 39 and thus the observation size is $39\times 2$. 
	
	\item Japanese vowels dataset (JAP). JAP contains 640 MTS observations of nine male speakers. Each observation is the utterance of two Japanese Vowels /ae/ from a speaker recorded by 12 LPC cepstrum coefficients (i.e. 12 variables). The task is to distinguish nine male speakers by their utterances.  Speakers 1–9 have numbers of observations: 61, 65, 118, 74, 59, 54, 70, 80, 59, respectively. The time length we use is 12 and the data dimension is $7\times 12$.
	
	\item WAFER dataset (WAF). WAF has 327 observations of two classes. Each observation is labeled as normal or abnormal and recorded by six vacuum-chamber sensors (i.e. variables) during monitoring an operational semiconductor fabrication plant. The normal and abnormal classes have 200 and 127 observations, respectively. The time length we use is 104 and thus the observation size is $104\times 6$.
	
	\item BCI dataset. BCI contains a total of 416 observations, which has been further divided into a training set of 316 observations and a test set of 100 observations. BCI consists of two classes: upcoming left and right hand movements, each 208 observations. Each observation is collected using 28 EEG channels (i.e. 28 variables). The time length is 500 and the dimension is $500\times 28$.
\end{enumerate}
\reft{tbl:stats} summarizes several statistics of the five datasets used in our experiments.
\begin{table}[htbp]
	\centering
	\caption{Several statistics for the five MTS data sets. $n$: Total number of observations; $d_1\times d_2$: matrix observation size; $c$: number of classes; $p_1$,  and $p_2$: small and large training proportions. }\label{tbl:stats}
	\begin{tabular}{ccccc}
		\toprule 		
		Dataset&$d_1\times d_2\times n$&$c$&$p_1$&  $p_2$\\ \midrule
		AUSLAN & $47\times22\times 625$& 25&1/9&1/4 \\
		ECG & $39\times2\times200$ & 2&1/10&4/5 \\
		JAPAN & $7\times12\times640$ & 9&1/20&4/5 \\
		WAFER & $104\times6\times327$ & 2&1/4&4/5 \\
		BCI & $500\times28\times316$ & 2&1/16&1 \\
		\bottomrule
	\end{tabular}
\end{table}

\subsection{Classification performance on real MTS datasets}
In this subsection, we compare the classification performance of RBLDA and related competitors including BLDA \cite{Noushath-2dlda}, PBLDA \cite{zhao2021-pblda}, RLDA \cite{friedman-rda}, BPCA \cite{zhangdq-2dpca} and DTW \citep{ruiz2020great}. For BLDA, we use the implementation in \cite{inoue-2dlda} as detailed in \refs{sec:blda}. For RBLDA, we use \refa{alg:rblda} and choose the two regularization parameters $r_1,r_2$ from the set $\{10^{-6}, 0.001, 0.01, 0.1,0.2,\dots,0.9,0.99\}$. For RLDA, we use a model selection algorithm similar to \refa{alg:rblda} and choose the regularization parameter $r$ from the same set as RBLDA. Let $\bV_w=(\bv_1,\bv_2,\dots,\bv_q)$ be the discriminant transformation obtained by RLDA with \refa{alg:rlda}. We also examine the performance of using  $\bV_w^u=(\bv_1^u,\bv_2^u,\dots,\bv_q^u)$, where $\bv_i^u=\bv_i/||\bv_i||_2$ (i.e., the length of $\bv_i^u$ equals to 1). For RLDA, we report the better result between $\bV_w$ and $\bV_w^u$. For BLDA, PBLDA and RBLDA, we also report similar results.

To measure the misclassification rate, each dataset is randomly split into a training set containing a ratio of $p_i,i=1,2$ of all samples per class and a test set containing the remaining samples. Several values of $p_i$ for each data set are investigated, which is listed in \reft{tbl:stats}. The small $p_1$ and large $p_2$ are respectively used to investigate the performance in small and large training sample size cases. For all methods, the 1-nearest-neighbor classifier is run in the reduced-dimensional space to obtain the test misclassification rates. All possible dimensionalities of the reduced representation are tried and the lowest average misclassification rate from 10 random splittings, its corresponding latent dimension, and standard deviation are reported. The regularization parameters involved in RLDA and RBLDA on each splitting are learned by 5-fold cross validation. For BCI, we always use the standard test set consisting of 100 observations to compute test misclassification rates. 

The classification results on the five MTS datasets are shown in \reft{tab:mis.mts}, where the `\textendash\textendash' sign means that the method fails to run, and hence the misclassification rate is not available. The detailed results with different dimensionalities in low-dimensional space are visualized in \reff{fig:err}. For RLDA, we plot the misclassification rates versus the dimensionality of $\by=\bV_w'\bx$. For the $\bY=\bV_1'\bX\bV_2$ in bilinear methods, we plot the misclassification rates versus the dimensionality of the vectorized $\vc(\bY)$. Since the size of $\bY$ on AUS or BCI is a little large, only an upper-left submatrix of $\bY$ that includes the optimal result is shown: $15\times 22$ for AUS, $16\times 25$ for BCI. For better comparison, we also add the result of DTW in \reff{fig:err} although it actually uses all the original data dimensions.

Denote the difference in misclassification rates between RBLDA and a \textit{method} by $\delta=e^{RPBLDA}-e^{method}$. To examine whether the difference $\delta$ is statistically significant, we perform the Wilcoxon signed-rank test at two levels: \textsl{Test1} and \textsl{Test2}. The null and alternative hypotheses are $H_0: \delta=0$ vs $H_1: \delta<0$\footnote{If the test is not rejected, a second test $H_0: \delta=0$ vs $H_1: \delta>0$ will be performed. If the second test is not yet rejected, it is concluded that the difference between two methods are not significant.}. \textsl{Test1} is used to test the significance at the level of each individual ($p_i$, $dataset_j$), $i=1,2,j=1,2,\dots,5$, using the 10 misclassification rates of a \textit{method} corresponding to the optimal dimension from 10 random splittings, i.e., $e_l^{method}, l=1,\dots, 10$. The results for \textsl{Test1} are also given in \reft{tab:mis.mts}. \textsl{Test2} is used to test the overall significance at the level of all ten individuals. This is performed by using the 100 misclassification rates of a \textit{method} collected from all individuals $(p_i, dataset_j)$, i.e, $e_l^{method}, i=l,\dots, 100$. For BLDA, only 80 misclassification rates are available for \textsl{Test2}. The results for \textsl{Test2} are summarized in \reft{tab:test.mts}. From \reft{tab:test.mts}, it can be seen that RBLDA is the overall winner of 10 competitions.
\begin{table*}[th]
	\centering
	\caption{\label{tab:mis.mts} The lowest average error rates shown as (mean$\pm$std(dim.)) and their corresponding dimensionalities by different methods on the five MTS datasets. The best method is shown in bold face. $^\bullet$ means that RBLDA is significantly better than the method, and $^\circ$ means that the difference between RBLDA and the method is not significant, using the Wilcoxon signed-rank test with a $p$-value of 0.05.}
		\resizebox{\linewidth}{!}{
			\begin{tabular}{cccccccc}
				\toprule
				\multirow{2}{*}{Data} & \multirow{2}{*}{$p$} &                                                                           \multicolumn{6}{c}{Method}                                                                            \\
				\cmidrule{3-8}     &                      &              BPCA              &             BLDA              &             PBLDA             &           RLDA            &          RBLDA          &           DTW            \\ \midrule
				\multirow{2}{*}{AUS}  &        $p_1$         &  15.5$\pm$2.0(1,21)$^\bullet$  &  5.2$\pm$2.5(1,16)$^\bullet$  &  5.2$\pm$2.5(1,16)$^\bullet$  & 7.7$\pm$2.4(14)$^\bullet$ & {\bf2.9}$\pm$0.7(1,13)  &  26.2$\pm$2.2$^\bullet$  \\
				&        $p_2$         &  11.0$\pm$1.4(1,21)$^\bullet$  &   1.9$\pm$0.5(1,15)$^\circ$   &   1.9$\pm$0.5(1,15)$^\circ$   &  2.3$\pm$1.1(16)$^\circ$  & {\bf1.7}$\pm$0.6(1,14)  &  19.4$\pm$1.7$^\bullet$  \\ \midrule
				\multirow{2}{*}{ECG}  &        $p_1$         &   25.5$\pm$4.8(15,1)$^\circ$   &    \textendash\textendash     &  36.6$\pm$7.9(1,2)$^\bullet$  &  23.6$\pm$6.5(1)$^\circ$  & {\bf23.6}$\pm$5.8(2,2)  &  27.0$\pm$4.0$^\bullet$  \\
				&        $p_2$         & {\bf13.4}$\pm$5.4(8,1)$^\circ$ &   16.3$\pm$7.0(2,2)$^\circ$   &   16.3$\pm$7.0(2,2)$^\circ$   & 20.7$\pm$3.8(1)$^\bullet$ &    14.9$\pm$5.3(2,1)    &  19.0$\pm$5.1$^\bullet$  \\ \midrule
				\multirow{2}{*}{JAP}  &        $p_1$         &  22.4$\pm$3.4(2,12)$^\bullet$  &   20.2$\pm$5.5(2,8)$^\circ$   &   20.2$\pm$5.5(2,8)$^\circ$   &  18.5$\pm$5.7(8)$^\circ$  & {\bf17.4}$\pm$4.1(1,12) &  20.5$\pm$2.3$^\bullet$  \\
				&        $p_2$         &  7.2$\pm$2.8(3,11)$^\bullet$   &   4.5$\pm$1.5(2,9)$^\circ$    &   4.5$\pm$1.5(2,9)$^\circ$    &  4.9$\pm$2.2(8)$^\circ$   & {\bf3.9}$\pm$1.4(3,12)  &  9.0$\pm$2.6$^\bullet$   \\ \midrule
				\multirow{2}{*}{WAF}  &        $p_1$         &  21.2$\pm$2.9(9,6)$^\bullet$   &   7.1$\pm$2.0(2,5)$^\circ$    &   7.1$\pm$2.0(2,5)$^\circ$    & 12.7$\pm$2.3(1)$^\bullet$ &    7.7$\pm$2.6(2,3)     & {\bf7.0}$\pm$1.1$^\circ$ \\
				&        $p_2$         &  11.1$\pm$3.6(12,6)$^\bullet$  & {\bf2.9}$\pm$1.6(4,5)$^\circ$ & {\bf2.9}$\pm$1.6(4,5)$^\circ$ & 9.6$\pm$4.3(1)$^\bullet$  &    3.0$\pm$2.0(6,4)     &   3.9$\pm$2.6$^\circ$    \\ \midrule
				\multirow{2}{*}{BCI}  &        $p_1$         &  48.3$\pm$2.4(1,8)$^\bullet$   &    \textendash\textendash     & 45.1$\pm$3.6(24,7)$^\bullet$  & 49.1$\pm$5.4(1)$^\bullet$ & {\bf42.5}$\pm$6.5(3,2)  &  48.8$\pm$3.8$^\bullet$  \\
				&        $p_2$         &  42.0$\pm$0.0(4,5)$^\bullet$   & 34.0$\pm$0.0(16,25)$^\bullet$ & 34.0$\pm$0.0(16,25)$^\bullet$ & 32.0$\pm$2.5(1)$^\bullet$ & {\bf25.7}$\pm$2.8(8,1)  &  47.0$\pm$0.0$^\bullet$  \\ \bottomrule
			\end{tabular}
		}
	\end{table*}
	\begin{table}[th]
		\centering
		\caption{\label{tab:test.mts} Results of $p$-values between RBLDA and other methods by the Wilcoxon signed-rank test. The superiority of RBLDA over a method is highly significant if $p$-value$<$ 0.001.}
		\begin{tabular}{ccccc}
			\toprule
			\multicolumn{5}{c}{Method} \\ \cmidrule{1-5}
			BPCA & BLDA & PBLDA & RLDA & DTW \\
			\midrule
			3.1e-12 & 3.6e-04 & 1.1e-5 & 3.0e-9 & 4.5e-13 \\
			\bottomrule
		\end{tabular}
	\end{table}	
	\begin{figure*}[htbp]
		\centering
		\scalebox{0.7}[0.58]{\includegraphics*{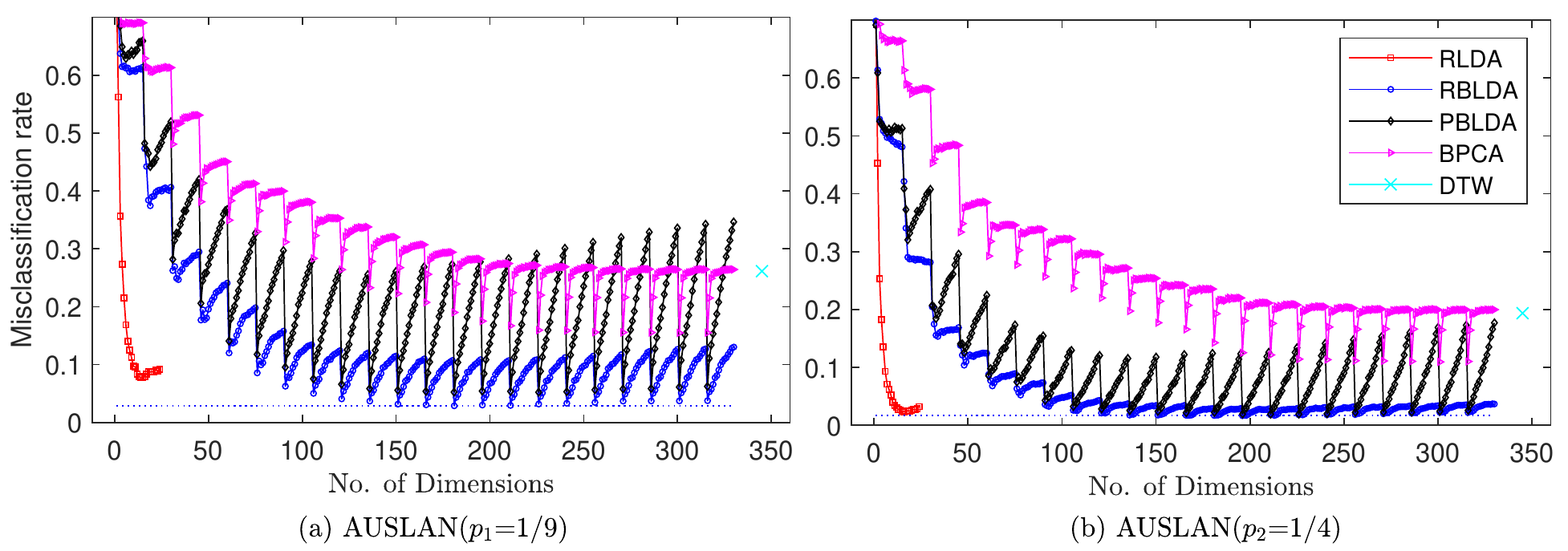}}
		\scalebox{0.7}[0.58]{\includegraphics*{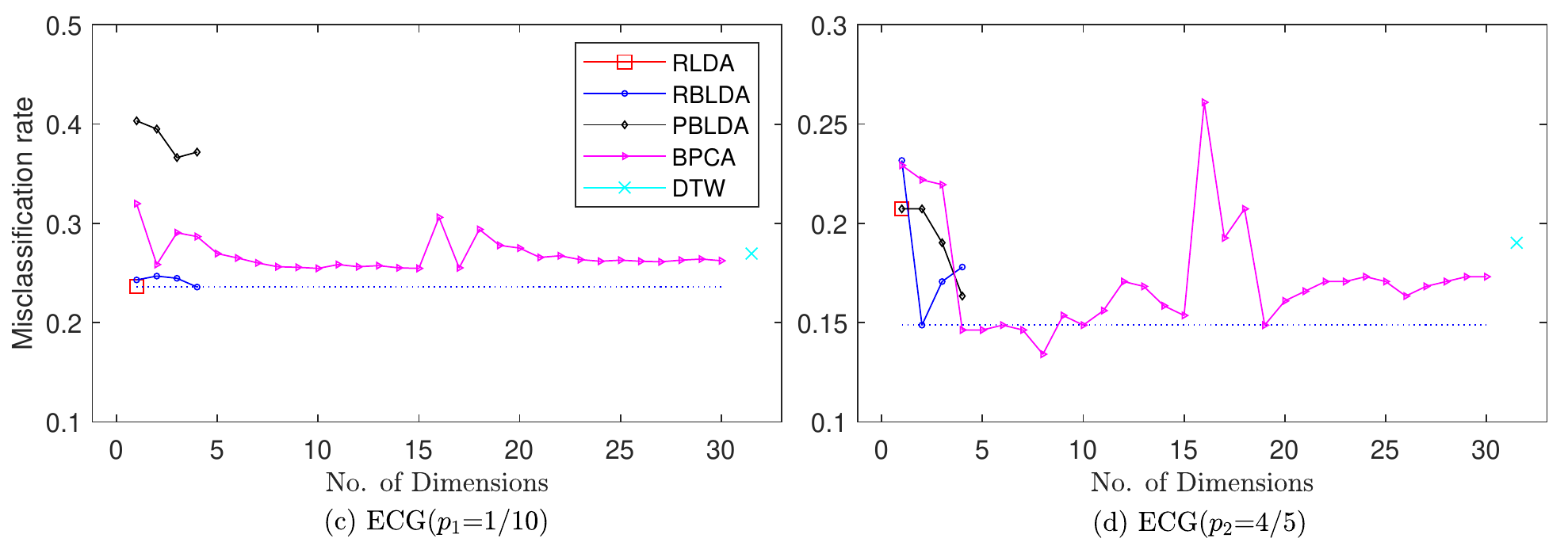}}
		\scalebox{0.7}[0.58]{\includegraphics*{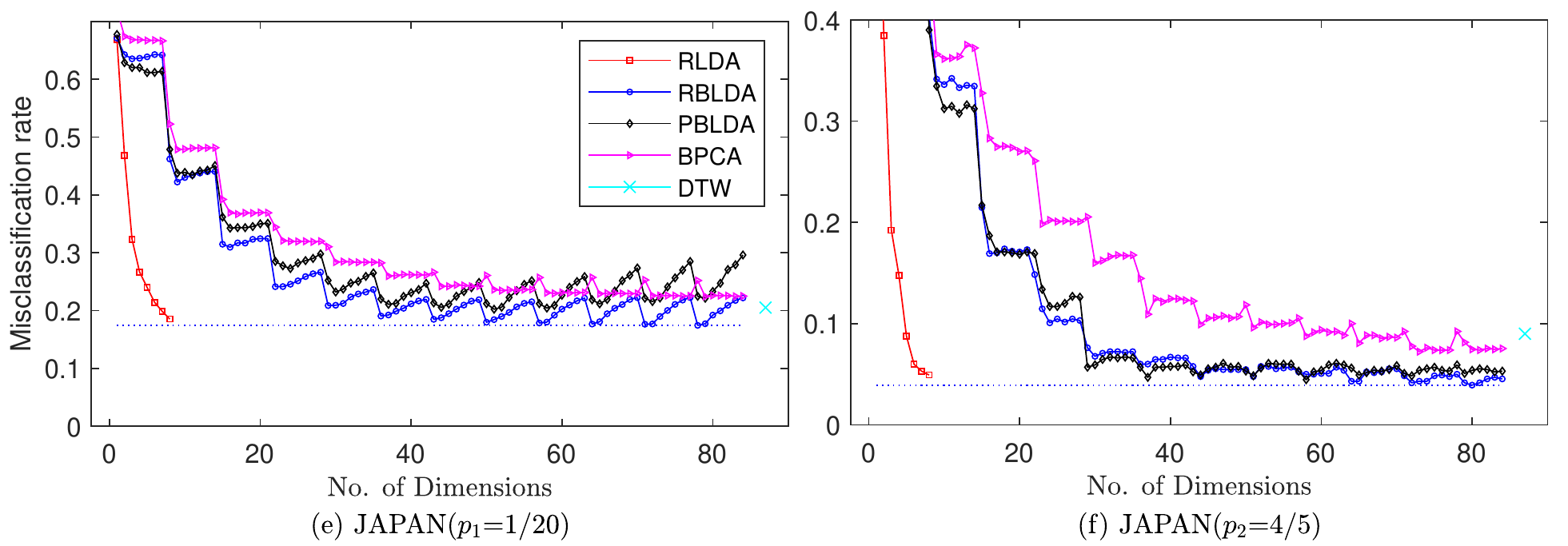}} 
		\scalebox{0.7}[0.58]{\includegraphics*{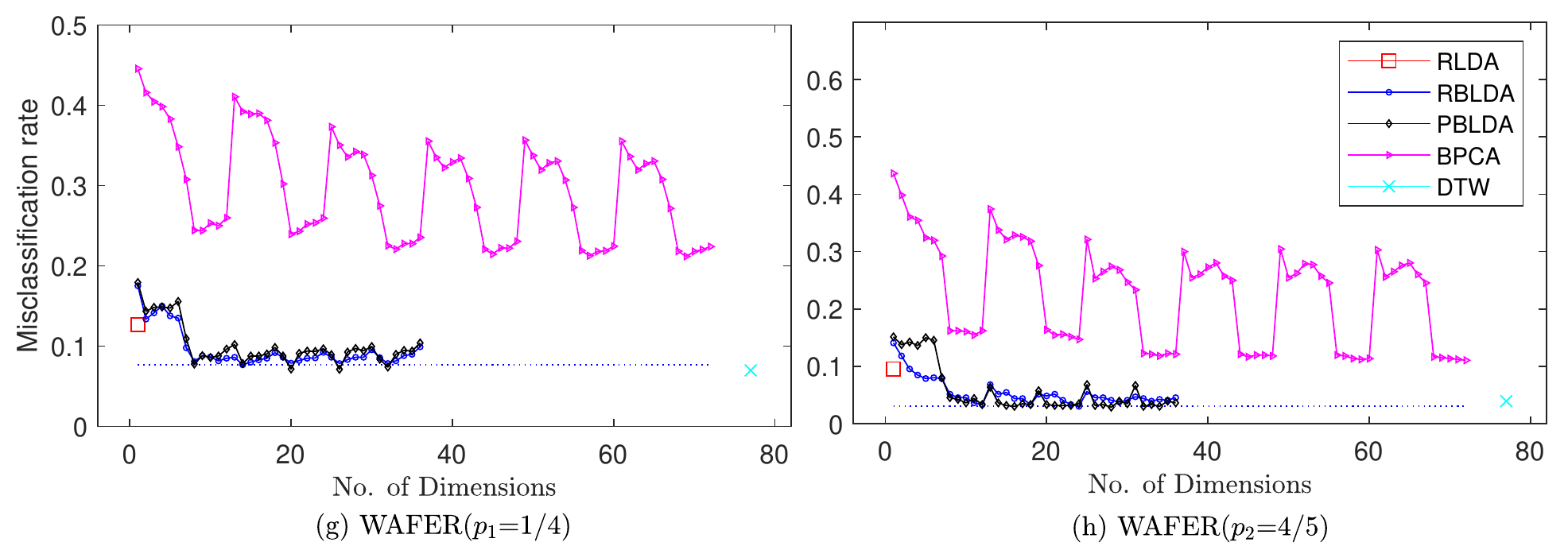}}
		\scalebox{0.7}[0.58]{\includegraphics*{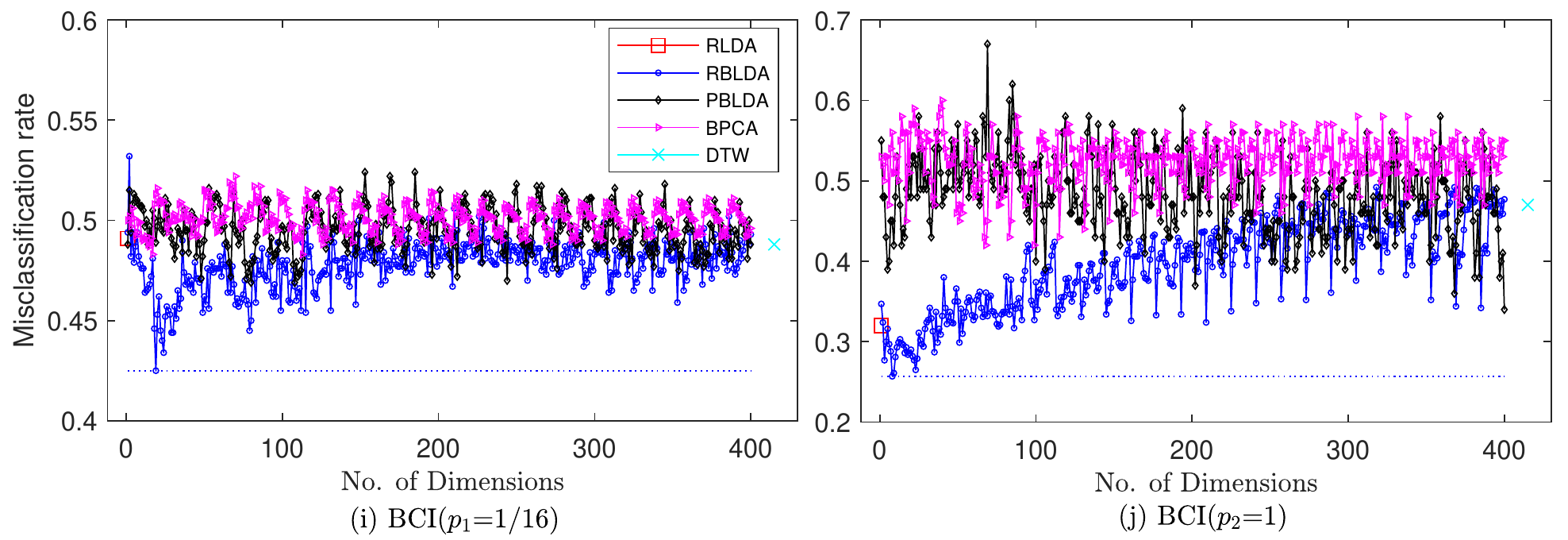}} 
		\caption{Misclassification rates versus number of dimensions by different methods in 10 competitions on the five datasets. Note that the marker x signals the result of DTW that actually uses all the original data dimensions.} \label{fig:err}
	\end{figure*}
	
	For individual comparison, the main observations from \reft{tab:mis.mts} and \reff{fig:err} include
	
	(i) RBLDA vs. RLDA. RBLDA performs better than RLDA, and 6 out of 10 competitions are significant, which means that the utilization of matrix data structure is useful for MTS data classification.
	
	(ii) RBLDA vs. BLDA and PBLDA. RBLDA is comparable with or superior to BLDA and PBLDA, and 4 out of 10 competitions are significant. RBLDA is more advantageous in small sample size cases on AUS, ECG, and BCI, where RBLDA has significantly better performance than PBLDA while BLDA even fails to run on ECG and BCI. Furthermore, it can be observed from \reff{fig:err} that, in most cases, RBLDA outperforms PBLDA not only in the optimal dimension but also in a wide range of dimensions. This indicates that the gain of applying regularization to BLDA is substantial.  
	
	(iii) RBLDA vs. BPCA. RBLDA is substantially better than BPCA except for $p_2$ on ECG data, which means that the utilization of label information is beneficial to MTS data classification in general.
	
	(iv) RBLDA vs. DTW. RBLDA significantly outperforms DTW on all five datasets except for WAFER. Their performance on WAFER are roughly comparable.
	%
	
	\subsection{Low-dimensional plots via RLDA and RBLDA}
	In this subsection, we compare the performance of RLDA and RBLDA for data visualization. To this end, we choose the two-class datasets ECG ($p=4/5$) and WAFER ($p=4/5$). The popularity of RLDA is partly due to the reduced-rank constraint that enables us to view informative low-dimensional projections of the data \citep{hastie2009elements}. Since $\rank(\bS_b)\leq c-1$, the dimensionality of RLDA subspace is at most $c-1$. If $d$ is much larger than $c$, the dimension reduction will be substantial. In particular, in the case that $c=2$, the dimensionality of the subspace is only 1. That is, ECG and WAFER can be viewed in and only in a one-dimensional RLDA subspace. 
	
	However, this is not the case for RBLDA. Since $\rank(\bS_{1b})\leq\min{(d_1,d_2(c-1))}$, $\rank(\bS_{2b})\leq\min$ $(d_1(c-1),d_2)$, the dimensions of RBLDA subspace is at most $2\times2$ on ECG and $6\times6$ on WAFER, respectively. This means that more discriminant features are available with RBLDA than those with RLDA. It is thus interesting to investigate whether or not the more discriminant features are useful for data visualization. For demonstration purpose, we use one result from 10 random spitting which is similar to the average one in \reft{tab:test.mts}. 
	
	In \reff{fig:low}, we plot the one-dimensional and two-dimensional projections of the remaining $1/5$ test data obtained by RLDA and RBLDA, respectively. In \reff{fig:low} (b) and (d), we choose only two discriminant features $y_{21}$ and $y_{22}$ from the $2\times2$ and $6\times6$ projected $\bY$, respectively. It can be seen from \reff{fig:low} that RBLDA yields better class separation than RLDA and more discriminant features extracted by RBLDA are beneficial to data visualization. From \reff{fig:err} (d) and (h), it can be observed that the more discriminant features available in RBLDA can improve the classification.
	\begin{figure*}[htb]
		\centering
		\scalebox{0.8}[0.6]{\includegraphics*{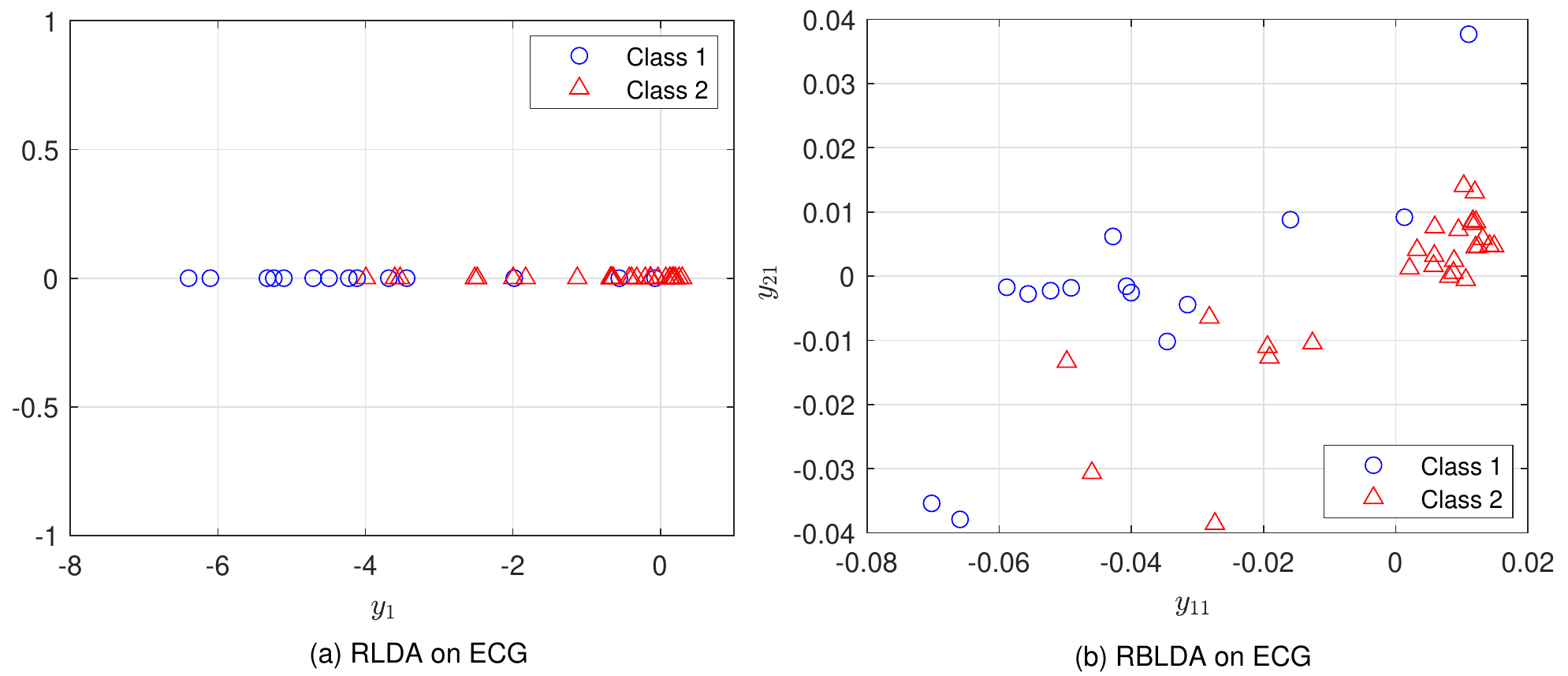}}
		\scalebox{0.8}[0.6]{\includegraphics*{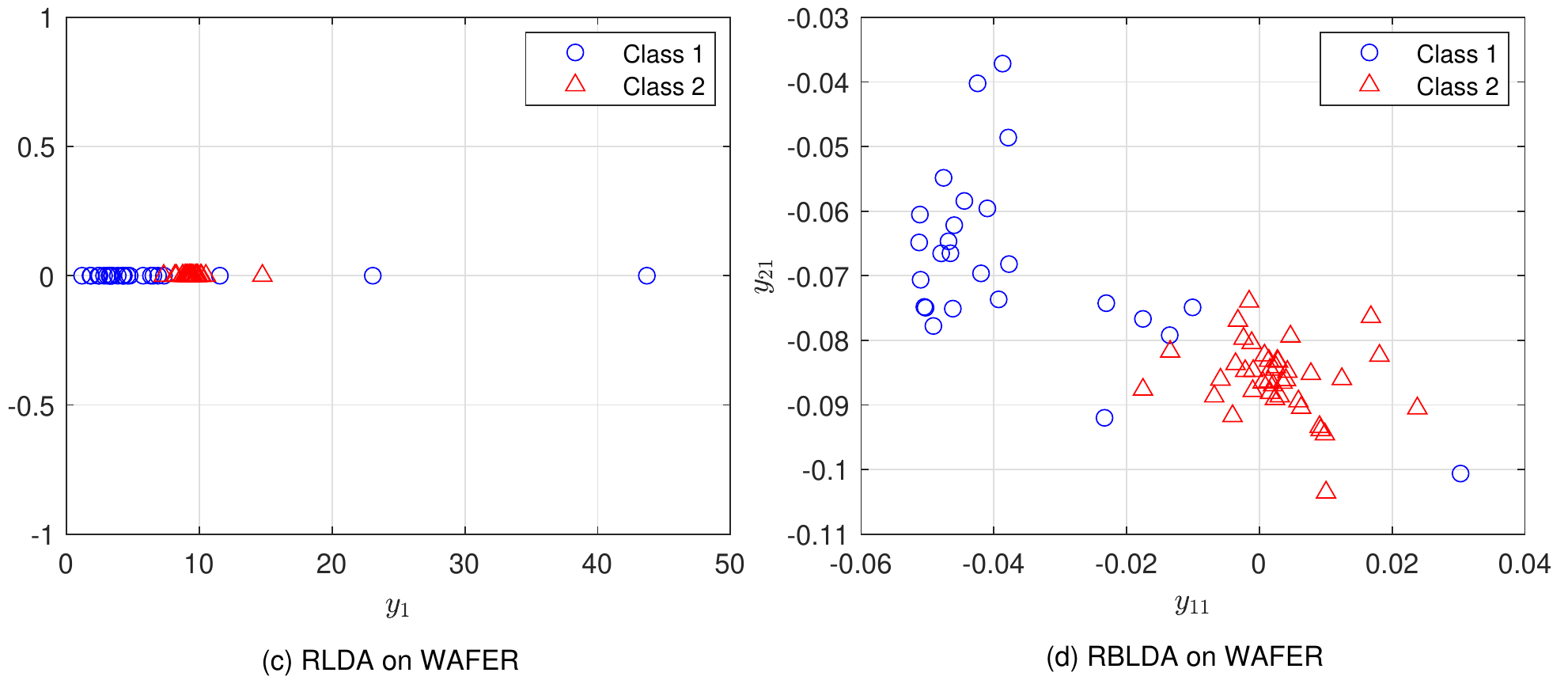}}
		\caption{} \label{fig:low}
	\end{figure*}
	\subsection{RBLDA model selection algorithm}
	In this subsection, we use BCI data to examine the efficiency of \refa{alg:rblda}. For this data, $d_1=500$ and $d_2=28$. To investigate the performance when $d_1$ is much greater than $d_2$, we construct another two datasets from BCI. We replicate the data along the row direction 8 times to form BCI2, yielding $d_1=4000$, and 32 times to form BCI3, yielding $d_1=16000$. We use the same training proportion $p=1/16$ for all the three datasets. For simplicity, we use the same number of candidates for $r_1$ and $r_2$, that is, $m=m_1=m_2$. Various values of $m$ are tried and the computational times on BCI, BCI2 and BCI3 are recorded in \reft{tab:time.bci}.
	\begin{table}[th]
		\centering
		\caption{\label{tab:time.bci} Comparison of time used by RBLDA model selection algorithm on BCI data with different dimensions.}
		\begin{tabular}{cccc}
			\toprule
			Number of 	&  \multicolumn{3}{c}{Dimension}\\ 	\cmidrule{2-4} 
			candidates $m^2$	& 500$\times$ 28 &4000$\times$ 28 &16000$\times$ 28  \\ \midrule
			1	&0.81&	1.73&	5.00 \\
			4	&0.95&	1.88& 	5.19 \\
			25	&1.23& 	2.52& 	6.27  \\
			100	&1.66& 	3.94& 	7.58  \\
			2500	&15.05& 18.45& 	27.05  \\
			10000	&57.05& 58.83& 	78.86 \\ \midrule
			$T(m,m)/T(1,1)$	&70.2& 33.9& 15.8 \\ \bottomrule
		\end{tabular}
	\end{table}
	
	It can be seen from \reft{tab:time.bci} that (i) the proposed RBLDA model selection algorithm is efficient. Although the candidate size increases by 10000 fold, the increases in running time are only 70.2 on BCI and 15.8 on BCI3; (ii) The proposed algorithm is more efficient when $d_1$ is much greater than $d_2$.
	
	\section{Conclusions and discussions}\label{sec:con}
	In this paper, we study discriminant analysis methods for MTS data classification and develop a new method called regularized bilinear discriminant analysis (RBLDA). To choose the complexity parameters involved, we further develop an efficient RBLDA model selection algorithm so that the cross validation procedure for RBLDA can be performed efficiently. The empirical results show that RBLDA significantly outperforms related competitors and the proposed model selection algorithm is efficient.   
	
	A distinct feature of MTS data is that the observations commonly have the same number of variables but may have different numbers of time points, namely the time length could be different. It would be interesting to extend our proposed RBLDA to accommodate such data, which is one of our future works.
	
	In recent years, robust extensions of 2DPCA and 2DLDA using various techniques have received much attention, including robust 2DLDA and BLDA using $L_p$-norm \citep{LI2019274,LI202173}, robust 2DLDA via Information Divergence \citep{zhang2020robust}, 2DPCA using the nuclear-norm \citep{zhang2015-n2dpca}, robust tensor LDA based on the Laplace distribution \citep{JU2021196}, 2DLDA using the nuclear-norm \citep{ZHANG201994}, trace ratio 2DLDA using $l_1$-norm \citep{LI2017216}, and etc. It would be worthwhile to extend these methods to MTS data classification in the future. 
	
	
	%
	\bibliography{journals,lit,jhzhao-pub}
	\bibliographystyle{elsarticle-num}
\end{document}

%% file: general_setting.tex
\newcommand{\bS}{\mathbf{S}}\newcommand{\bB}{\mathbf{B}}\newcommand{\bM}{\mathbf{M}}
\newcommand{\bF}{\mathbf{F}}\newcommand{\bG}{\mathbf{G}}\newcommand{\bX}{\mathbf{X}}
\newcommand{\bT}{\mathbf{T}}\newcommand{\bR}{\mathbf{R}}\newcommand{\bY}{\mathbf{Y}}\newcommand{\bU}{\mathbf{U}}
\newcommand{\bA}{\mathbf{A}}

\newcommand{\btV}{\widetilde{\mathbf{V}}}

\newcommand{\bI}{\mathbf{I}}\newcommand{\bE}{\mathbf{E}}
\newcommand{\bbE}{\mathbb{E}}\newcommand{\bbI}{\mathbb{I}}\newcommand{\bW}{\mathbf{W}}\newcommand{\bH}{\mathbf{H}}
\newcommand{\cC}{\mathcal{C}}\newcommand{\cF}{\mathcal{F}}\newcommand{\cL}{\mathcal{L}}\newcommand{\bLmd}{\mathbf{\Lambda}}
 
\newcommand{\bV}{\mathbf{V}}

\newcommand{\bo}{\mathbf{0}}
\newcommand{\bx}{\mathbf{x}}\newcommand{\bm}{\mathbf{m}}
\newcommand{\by}{\mathbf{y}}

\newcommand{\bv}{\mathbf{v}}

 \newcommand{\bPi}{\boldsymbol{\Pi}}\newcommand{\bGam}{\boldsymbol{\Gamma}}
\newcommand{\bpi}{\boldsymbol{\pi}}

\newcommand\refe[1]{(\ref{#1})}\newcommand\refp[1]{Proposition~\ref{#1}}
\newcommand\reff[1]{Fig.~\ref{#1}}\newcommand\reft[1]{Tab.~\ref{#1}}\newcommand\refs[1]{Sec.~\ref{#1}}
\newcommand\refa[1]{Algorithm~\ref{#1}}  
 
\theoremstyle{plain} \newtheorem{prop}{Proposition} 

\def\R{{\mathbb R}}

\def\tr{\mbox{tr}}
\def\vc{\mbox{vec}}

\def\diag{\mbox{diag}}
\def\rank{\mbox{rank}}
\def\bdiag{\mbox{blkdiag}}
\def\etal{{\em et al.\/}\,}

\def\s2{\sigma^2}

\mathchardef\mh="2D

\definecolor{purple}{RGB}{148,0,211}
